\documentclass[10pt,twocolumn,letterpaper]{article}

\usepackage{cvpr}              

\usepackage{url}
\usepackage{booktabs}       
\usepackage{amsfonts}       
\usepackage{nicefrac}       
\usepackage{microtype}      
\usepackage{xcolor}         
\usepackage{multicol}

\usepackage{graphicx}
\usepackage{amsmath}
\usepackage{bm}
\usepackage{color}
\usepackage{amssymb}
\usepackage{bm}
\usepackage{makecell}
\usepackage{threeparttable}
\usepackage{multirow}
\usepackage{caption}
\usepackage{subcaption}
\usepackage{algorithm}
\usepackage{algorithmic}
\usepackage{natbib}

\newtheorem{theorem*}{Theorem}
\newtheorem{theorem}{Theorem}

\newtheorem{proof}{proof}

\definecolor{cvprblue}{rgb}{0.21,0.49,0.74}
\usepackage[pagebackref,breaklinks,colorlinks,allcolors=cvprblue]{hyperref}

\def\paperID{9715} 
\def\confName{CVPR}
\def\confYear{2025}

\title{Inference-Scale Complexity in ANN-SNN Conversion for High-Performance and Low-Power Applications \vspace{-10pt} }

\author{Tong Bu\textsuperscript{1\dag}, Maohua Li\textsuperscript{2\dag}, Zhaofei Yu\textsuperscript{1\rm *} \\
{\small \textsuperscript{\dag} Equal contribution, \textsuperscript{*} Corresponding author}\\
{\tt\small putong30@pku.edu.cn, maohuali@hhu.edu.cn, yuzf12@pku.edu.cn}
}
\begin{document}
\maketitle
\renewcommand{\thefootnote}{}
\footnote{\textsuperscript{1} Institute for Artificial Intelligence, Peking University}
\footnote{\textsuperscript{2} School of Artificial Intelligence and Automation, Hohai University}

\begin{abstract}
Spiking Neural Networks (SNNs) have emerged as a promising substitute for Artificial Neural Networks (ANNs) due to their advantages of fast inference and low power consumption. However, the lack of efficient training algorithms has hindered their widespread adoption. Even efficient ANN-SNN conversion methods necessitate quantized training of ANNs to enhance the effectiveness of the conversion, incurring additional training costs. To address these challenges, we propose an efficient ANN-SNN conversion framework with only inference scale complexity. The conversion framework includes a local threshold balancing algorithm, which enables efficient calculation of the optimal thresholds and fine-grained adjustment of the threshold value by channel-wise scaling. We also introduce an effective delayed evaluation strategy to mitigate the influence of the spike propagation delays. We demonstrate the scalability of our framework in typical computer vision tasks: image classification, semantic segmentation, object detection, and video classification. Our algorithm outperforms existing methods, highlighting its practical applicability and efficiency. Moreover, we have evaluated the energy consumption of the converted SNNs, demonstrating their superior low-power advantage compared to conventional ANNs. This approach simplifies the deployment of SNNs by leveraging open-source pre-trained ANN models, enabling fast, low-power inference with negligible performance reduction. Code is available at 
\href{https://github.com/putshua/Inference-scale-ANN-SNN}{https://github.com/putshua/Inference-scale-ANN-SNN}.
\end{abstract}

\section{Introduction}
\begin{figure}[t]
\centering
\includegraphics[width=0.48\textwidth]{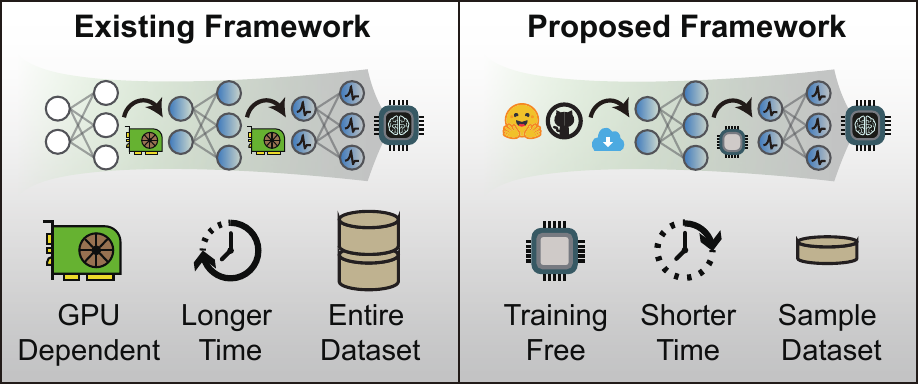}
\caption{Compared to existing frameworks that require retraining a quantized ANN, the proposed framework is able to directly convert a pre-trained ANN to an SNN at inference-scale complexity, significantly reducing computational requirements, minimizing dependence on GPUs, and requiring only a small subset of the original dataset.}
\label{fig:cw}
\vspace{-5pt}
\end{figure}
Recent advancements in large models have significantly reshaped the landscape of deep learning technology, industry, and the research community. These advanced models, characterized by their massive scale and unprecedented zero-shot generalizability in downstream tasks, greatly impact our daily lives. Nevertheless, the pursuit of larger models raises concerns about high energy consumption for model inference and training. The deployment of large models on resource-constrained devices has also become a challenge. 

Spiking neural networks (SNNs), a well-established type of neural network that mimic the function of biological neurons and are considered the third generation of ANNs~\citep{maas1997networks}, offer a potential solution to these issues as an alternative to Artificial Neural Networks (ANNs). 
Unlike ANNs, spiking neurons encode information as discrete events (binary spikes or action potentials) and generate outputs based on both their inputs and internal state. This unique processing mechanism enables SNNs to effectively handle spatio-temporal information while leveraging sparse representations, leading to more efficient and biologically plausible information processing.

With the development of the neuromorphic computing hardware~\citep{pei2019towards, debole2019truenorth, davies2018loihi, nicolas2021sparse, fang2020spike, zenke2021brain, yao2024spike}, SNNs can now be deployed on neuromorphic chips and further applied in power-limited scenarios~\citep{chen2022state,shen2023ESL}. However, the lack of efficient training algorithms has hindered their widespread application. Recent learning methods have made significant progress in training deep convolutional SNNs~\citep{duantemporal, huhigh} and large spiking transformers using supervised learning algorithms~\citep{zhu2023spikegpt, shi2024spikingresformer, yaospike}. 
While supervised training enables SNNs to achieve performance comparable to ANNs with the same architecture and reduces inference time, the memory and time costs for training scale linearly with the number of inference time-steps, making it impractical for training large energy-efficient models.
A more feasible approach is ANN-to-SNN conversion~\citep{cao2015spiking}, which transforms pre-trained ANNs into SNNs with minimal computational overhead. However, converted SNNs typically require more inference time-steps to match the performance of their ANN counterparts, leading to increased latency and higher energy consumption. A potential solution is to re-train a modified ANN before conversion, which can significantly enhance SNN performance while reducing inference time-steps~\citep{hwang2021low, bu2022optimized}. Although retraining-based conversion methods generally improve performance across most tasks, the additional computational cost of re-training poses a notable overhead.

In this paper, we propose an efficient post-training ANN-SNN conversion algorithm that directly obtains high-performance SNNs from pre-trained ANNs with inference-scale complexity. Figure~\ref{fig:cw} illustrates the key difference between existing ANN-to-SNN conversion frameworks and our proposed conversion framework. Compared to training-required methods, our approach greatly reduces the need for GPU resources, requires fewer data samples, and shortens the overall conversion process. 
We employ a local-learning-rule based algorithm for rapid conversion from ANNs to SNNs, avoiding the requirements for computationally expensive backpropagation-based training or fine-tuning. Additionally, we introduce a delayed evaluation strategy to further enhance performance. This conversion framework simplifies the deployment of SNNs by leveraging open-source, pre-trained ANN models, allowing direct application of converted SNNs on neuromorphic hardware with only inference-level computational costs. 
Overall, the proposed conversion framework offers two potential advantages. First, the lightweight local learning algorithm reduces dependence on GPUs, enabling online mapping from ANN to SNN on neuromorphic chips. Second, this approach reduces the need for extensive retraining of models, making it highly suitable for the practical deployment of large models in resource-constrained environments. 
Our contribution can be summarized as follows:
\begin{itemize}
    \item We introduce an inference-scale complexity ANN-SNN conversion framework, enabling the conversion of pre-trained ANN models into high-performance SNNs while maintaining computational costs at inference level or even lower.
    \item We propose an efficient ANN-SNN conversion framework, theoretically deriving an upper bound for the conversion error between the original ANN and the converted SNN. To minimize this error bound, we introduce a light-weight local learning method for threshold optimization. Additionally, we design a delayed evaluation strategy to mitigate the influence of spike propagation delays.
    \item We experimentally demonstrate that our method is a light-weight, plug-and-play solution capable of integrating with various conversion methods for further performance improvement. Our framework successfully scales to three typical computer vision tasks, outperforming existing ANN-SNN conversion methods in classification tasks while also demonstrating its practical applicability in handling regression-based vision tasks.
\end{itemize}

\section{Related Works}
Researchers have endeavored to explore effective SNN training approaches for decades. Early studies mainly focused on the unsupervised Spike-Timing-Dependent Plasticity (STDP) algorithm~\citep{hebb2005organization, caporale2008spike}. These studies aimed to utilize the STDP algorithm for shallow networks and feature extractors~\citep{song2000competitive, masquelier2007unsupervised, wade2010swat, diehl2015unsupervised, kheradpisheh2018stdp}. 
In parallel, supervised training approaches have been developed to achieve high-performance SNNs. \citet{bohte2000spikeprop} were pioneers in applying the backpropagation algorithm to train SNNs, introducing a temporal-based backpropagation method that uses timestamps as intermediate variables during gradient computation. \citet{wu2018STBP} treated SNNs as recurrently connected networks and utilized Back Propagation Through Time (BPTT) for training, termed spatial-temporal backpropagation (STBP). Unlike these methods, \citet{cao2015spiking} proposed an ANN-SNN conversion method that achieves high-performance SNNs from pre-trained ANNs. Currently, research efforts are primarily focused on enhancing the efficiency of training algorithms for high-performance SNNs, leveraging the three methods mentioned above.

Based on whether global backpropagation is required, SNN learning methods can be ascribed into post-training and training-involved approaches. 
Training-required methods, including direct training of SNNs \citep{fang2020spike, wu2018STBP, shrestha2018slayer, neftci2019surrogate, wu2019direct, fang2021incorporating, jiangtab, yang2022training}, re-training before ANN-SNN conversion \citep{bu2021optimal, deng2020optimal, li2024seenn, jiang2023unified, hao2023bridging, huang2024towards, wu2024ftbc} and fine-tuning of converted SNNs~\citep{rathi2020enabling, rathi2021diet, li2024error, bojkovic2024data}, typically require the participation of the back-propagation algorithms, making them more computationally intensive and resource-demanding. 
Post-training learning methods mainly refer to conversion-based algorithms that directly convert the pre-trained ANN into SNN. \citet{cao2015spiking} pioneered this approach by mapped the weights of a lightweight CNN network to an SNN, demonstrating significant energy efficiency improvements through hardware analysis.
\citet{diehl2015fast} proposed a weight and threshold balancing method to mitigate performance degradation, while 
\citet{hunsberger2015spiking} further extended conversion-based learning to LIF neurons. 
\citet{rueckauer2017conversion} established a theoretical framework for conversion-based algorithms, leading to further optimizations such as robust normalization, an enhanced threshold balancing method designed to optimize the trade-off between inference latency and accuracy~\citep{sengupta2019going}. 
\citet{han2020deep} extended ANN-to-SNN conversion to temporal coding SNNs, improving accuracy and introducing a reset-by-subtraction mechanism instead of the traditional reset-to-zero, thereby preventing information loss~\citep{Han_2020_CVPR}.
Beyond classification tasks, conversion-based methods have been successfully applied to regression problems such as object detection, with \citet{kim2020spiking} demonstrating the feasibility of spike-based object detection models. Further optimizations include adding an initial membrane potential for improved accuracy and reduced latency~\citep{bu2022optimized, hao2023reducing}, as well as employing a lightweight conversion pipeline to optimize thresholds and introduce a bias term at each inference step~\citep{li2021free}. 
Although performance is limited, those light-weight training methods are able to acquire SNNs from pre-trained ANNs with negligible cost of multiple iterations of inference.

\section{Preliminaries}
\subsection{Neuron Models}
The core components of conventional ANNs are point neurons, where their forward propagation can be represented as a combination of linear transformations and non-linear activations, that is 
\begin{align}
    \bm{a}^l=A\left(\bm{w}^l \bm{a}^{l-1}\right). \label{eq:relu}
\end{align}
Here, $\bm{w}^l$ is the weights between layer $l-1$ and layer $l$, $\bm{a}^l$ is the activation vector in layer $l$, and $A(\cdot)$ denotes the non-linear activation function.

The neurons in SNNs are spiking neurons with temporal structures. Like most ANN-SNN conversion methods ~\citep{cao2015spiking, bu2021optimal, diehl2015fast, sengupta2019going, brette2007simulation, deng2020optimal}, we employ the Integrate-and-Fire (IF) neuron ~\citep{izhikevich2004model,gerstner2014neuronal} in this paper. 
For computational convenience, we use the following discrete neuron function:
\begin{align}
\bm{v}^l(t^-)&=\bm{v}^l({t-1})+\bm{w}^{l}\bm{s}^{l-1}(t)\theta^{l-1}, \label{21}\\ 
\bm{s}^l(t)&= H(\bm{v}^l(t^-)-\theta^l),\\
\bm{v}^l(t)&= \bm{v}^l({t}^-)-\theta^l \bm{s}^l(t). \label{23}
\end{align}
Here $\bm{s}^l(t)$ describes whether neurons in layer $l$ fire at time-step $t$, with firing triggered when the membrane potential before firing $\bm{v}^l(t^-)$ reaches the firing threshold $\theta^l$, as represented by the Heaviside step function $H(\cdot)$. 
To alleviate information loss, we adopt the "reset-by-subtraction" mechanism~\citep{Han_2020_CVPR, Bodo2017Conversion} for updating the membrane potential $\bm{v}^l(t)$ after spike emission. Specifically, instead of resetting to the resting potential, the membrane potential is reduced by $\theta^l$ after emitting a spike.




\subsection{ANN-SNN Conversion}
The pivotal idea of ANN-SNN conversion is to map the activation of analog neurons to the postsynaptic potential (or firing rate) of spiking neurons. Specifically, by defining the intial membrane potential as $\bm{v}^{l}(0)$ and accumulating the neuron function (Equations~\ref{21} and \ref{23}) from 1 to T and dividing both sides by time-step T, we obtain:
\begin{align}
    \frac{\bm{v}^{l}(T)-\bm{v}^{l}(0)}{T}=\frac{\sum_{t=1}^{T} \bm{w}^{l} \bm{s}^{l-1}(t) \theta^{l-1} - \sum_{t=1}^{T}  \bm{s}^{l}(t)\theta^{l}}{T} . \label{eq:psp}
\end{align}
This equation reveals a linear relationship between the average firing rate of neurons in adjacent layers by defining average postsynaptic potential as $\bm{r}^l(t) = \sum_{i=1}^t \bm{s}^l(i) \theta^l /t$:
\begin{align}
    \bm{r}^l(T)=\bm{w}^l \bm{r}^{l-1}(T)-\frac{\bm{v}^l(T)-\bm{v}^l(0)}{T}. \label{eq:pot}
\end{align}
As $T \rightarrow +\infty$, the conversion error can generally be assumed to approach zero. Therefore, using the equations presented, a trained ANN model can be converted into an SNN by replacing ReLU neurons with IF neurons, which forms the fundamental principle of ANN-SNN conversion. Since Equations~\ref{eq:relu} and~\ref{eq:pot} are not mathematically identical, some conversion error typically remains. 
To mitigate these errors, threshold balancing~\citep{diehl2015fast} and weight scaling algorithms~\citep{li2021free, sengupta2019going} play a crucial role. Both techniques aim to adjust synaptic weights or neuron thresholds according to the distribution range of neuron inputs, effectively reducing clipping errors. Previous studies have shown that weight scaling and threshold balancing are equivalent in effect and can achieve high-performance SNNs~\citep{bu2022optimized}.



\section{Methods}
In this section, we present our post-training ANN-SNN conversion framework. We begin by defining the conversion error and deriving its upper bound. After that, we propose a local threshold balancing method based on a local learning rule. Additionally, we introduce the channel-wise threshold balancing technique, the pre-neuron max pooling layer, and the delayed evaluation strategy. These techniques are seamlessly integrated into a comprehensive conversion framework to enhance both performance and efficiency.

\subsection{Conversion Error Bound}
We follow the conversion error representation from \citep{bu2021optimal} and define the conversion error $e^l$ in layer $l$ as the 2-norm of the difference between the postsynaptic potential of neurons in SNNs and the activation output of neurons in ANNs, that is 
\begin{align}
    e^l &= \lVert \text{S}(\bm{\hat{z}}^l; \theta^l) - \text{A}(\bm{z}^l)\rVert_2.
\end{align}
Here, $\bm{\hat{z}}^l = \bm{w}^l \bm{r}^{l-1}$ denotes the average input current from layer $l-1$ in SNNs, and $\bm{z}^l = \bm{w}^l \bm{a}^{l-1}$ represents the activation input from layer $l-1$ in ANNs. Both the ANN and SNN models share the same weights, denoted as $\bm{w}^l$. ${\text{A}}(\bm{z}^l)$ represents the output of nonlinear ReLU activation function while ${\text{S}}(\bm{\hat{z}}^l; \theta^l)=\bm{r}^l(T)$ represents the average output synaptic potential at time-step $T$ with given average input $\bm{\hat{z}}^l$. 

Since our primary objective is to minimize the error of the final outputs, we define the conversion error between the ANN and SNN models as the layer-wise conversion error in the last layer $L$, denoted as $e_\text{model} = e^L$. A straightforward approach would be to directly use the conversion error between models as a loss function and optimize it, which can be viewed as a variant of knowledge distillation \citep{yang2022training}. However, this method is computationally expensive in terms of both time and memory, resembling the challenges of supervised training of SNN.
To address this, we scale the conversion error and derive an error upper bound, thereby simplifying the optimization process. We formalize this approach in the following theorem.


\begin{theorem}
The layer-wise conversion error can be divided into intra-layer and inter-layer errors:
\begin{align}
    e^l &\leqslant \overbrace{\lVert \textnormal{S}(\bm{\hat{z}}^l; \theta^l) - \textnormal{A}(\bm{\hat{z}}^l)\rVert_2}^\textnormal{intra-layer~error} + \overbrace{\lVert \bm{w}^l \rVert_2 e^{l-1}}^\textnormal{inter-layer error}.
\end{align}
Given that both ANN and SNN models receive the same input in the first layer, leading to $e^0 = 0$, the upper bound for the conversion error between ANN and SNN models in an $L$-layer fully-connected network is given by
\begin{align}
    e_\textnormal{model}=e^L &\leqslant \sum_{l=1}^{L} \left (\prod_{k=l+1}^{L} \lVert \bm{w}^k \rVert_2\right ) \left \lVert \textnormal{S}(\bm{\hat{z}}^l; \theta^l) - \textnormal{A}(\bm{\hat{z}}^l) \right \rVert_2
\end{align}
\end{theorem}
Here $\lVert \bm{w}^k \rVert_2$ is the matrix norm (p=2) or spectral norm of the weight matrix.
The detailed proof is provided in the Appendix. Theorem 1 indicates that the conversion error is bounded by the weighted sum of errors across all layers. Based on this result, we introduce the local threshold balancing algorithm in the following section.

\subsection{Local Threshold Balancing Algorithm}

The target of ANN-SNN conversion is to minimize the conversion error and achieve high-performance SNNs.  An alternative approach is to optimize the error bound, defined as:
\begin{equation}
\begin{split}
     \arg & \min_{\bm{\theta}} \mathbb{E}_{x^0 \in \mathcal{D}} \\
     & \left[ \sum_{l=1}^{L} \left (\prod_{k=l+1}^{L} \lVert \bm{w}^k \rVert_2\right ) \left \lVert \text{S}(\bm{\hat{z}}^l; \theta^l) - \text{A}(\bm{\hat{z}}^l) \right \rVert_2 \right ].
\end{split}
\end{equation}
Here $\bm{\theta}$ represents the set of threshold values to be optimized across all layers, i.e., $\bm{\theta}=\{\theta^1,\theta^2,...,\theta^L\}$. $x^0$ denotes the input data sample drawn from the dataset $\mathcal{D}$. This approach leverages data-driven conversion, minimizing the expectation of the conversion error bound over the data distribution.  To reduce computational costs during optimization, we introduce two key simplifications. Firstly, we use a greedy strategy to optimize the threshold $\bm{\theta}$ layer by layer,
that is 
\begin{equation}
\begin{split}
    \text{For eac}&\text{h layer}~l,~\arg \min_{\theta^l} \mathbb{E}_{x^0 \in \mathcal{D}}\\
    & \left[ \left (\prod_{k=l+1}^{L} \lVert \bm{w}^k \rVert_2\right ) \left \lVert \text{S}(\bm{\hat{z}}^l; \theta^l) - \text{A}(\bm{\hat{z}}^l) \right \rVert_2 \right ].
\end{split}
\end{equation}
Secondly, we simplify the function $\text{S}(\cdot; \theta)$ by neglecting its temporal dynamics and approximating it with the clipping function $\text{C}(\cdot; \theta)$, as previous works have demonstrated that spiking neuron and clipping functions behave equivalently given sufficient number of time-steps~\citep{deng2020optimal}. The average input in the SNN can therefore be estimated as $\bm{\hat{z}}^0 = \bm{z}^0; \bm{\hat{z}}^{l+1} = \bm{w}^{l+1} \text{C}(\bm{\hat{z}}^l; \theta^l)$. Furthermore, We utilize the squared norm in the objective function for smoother optimization. The optimization problem then becomes:
\begin{equation}
\begin{split}
    \text{For eac}&\text{h layer}~l,~ \arg \min_{\theta^l} \mathbb{E}_{x^0 \in\mathcal{D}} \\
    & \left[ \left (\prod_{k=l+1}^{L} \lVert \bm{w}^k \rVert_2\right ) \lVert \text{C}(\bm{\hat{z}}^l;\theta^l) - \text{A}(\bm{\hat{z}}^l)\rVert_2^2\right],
\end{split}
\label{eq:aim}
\end{equation}
\begin{equation}
    \text{where}~\text{C}(x; \theta) = \min (\max(0, x), \theta).
\end{equation}
Intuitively, we can employ the stochastic gradient descent algorithm to optimize the threshold. Since $\prod_{k=l+1}^{L} \lVert \bm{w}^k \rVert_2$ can be considered as a constant when weights are fixed for each layer, this term can be absorbed into the learning rate. The final update rule for the local threshold balancing algorithm at each step is:
\begin{align}
    \Delta \theta^l &=  - \sum_{i=0}^{N-1} 2(\hat z^l_i-\theta^l)H(\hat z^l_i-\theta^l), \label{eq:grad} \\
    \theta^l &\leftarrow \theta^l - \eta \Delta \theta^l.
\end{align}
Detailed derivation is provided in the Appendix. Here $\Delta \theta^l$ denotes the step size during optimization and $\eta$ is the learning rate or update step size. $N$ is the number of elements in the input vector, and $\hat{z}_i$ denotes each element in the vector $\bm{\hat z}^l$. $H(\cdot)$ is the Heaviside step function.

In practice, as shown in Figure~\ref{fig:lws}, we jointly optimize the threshold for each layer by sampling a small set of image samples from the training dataset. Before starting the conversion process, we first replace all ReLU activation functions $\text{A}(\cdot)$ with clipping functions $\text{C}(\cdot)$ and initialize the threshold to zero. We then randomly sample data batches from the training dataset and perform model inference with these data samples. The thresholds are locally updated during the inference until convergence. Notably, the parameter $\theta^l$ is updated locally, ensuring that the computational cost of each iteration is similar to the inference of the ANN model and the computational cost of the whole conversion process is at inference level.

\begin{figure}[t]
\centering
\includegraphics[width=0.45\textwidth]{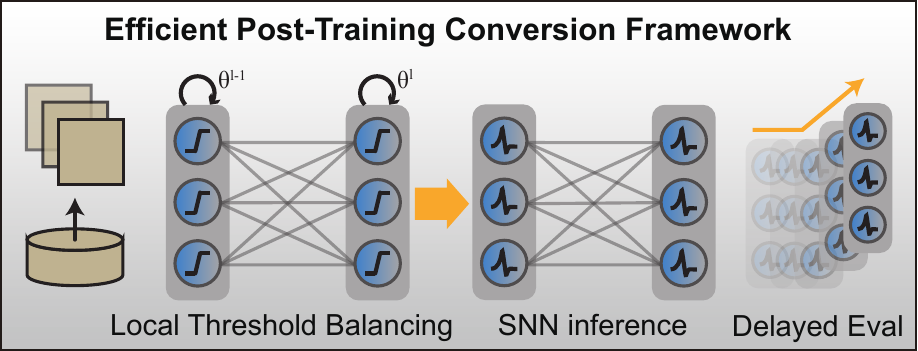}
\caption{Illustration of the proposed conversion framework. It requires only a small subset of data for conversion using local threshold balancing. During inference, the delayed evaluation technique is employed to enhance the accuracy of output estimation.}
\label{fig:lws}
\end{figure}
\vspace{-5pt}

\subsection{Pre-Neuron Max Pooling Layer}
The conversion of the max pooling layer presents a significant challenge. As a pivotal feature in the process of down-sampling input representations, the max pooling layer is a fundamental component in convolutional neural networks and is widely used in most convolutional architectures. However, due to the binary nature of spiking neurons, max pooling layers cannot effectively perform downsampling in feature maps. 
Typically, in a single max-pooling operation, multiple elements may share the same value, preventing the extraction of the most salient features. Consequently, max pooling is often avoided as a downsampling layer in SNNs. In most previous ANN-SNN methods, the max pooling layer in ANN is often replaced by an average pooling layer, and the model is re-trained or fine-tuned before conversion.

To convert architectures that contain max pooling layers, we propose a simple yet effective method. We replace all max-pooling layers with neuron layers, allowing the pre-neuron floating-point input to pass through the downsampling layer before being input into the neurons. We have the following theorem:
\begin{theorem}
\label{t2}
The order of the max pooling layer and ReLU layer does not affect the ANN output results.
\begin{align}  
    \max_i \textnormal{R}(\bm{z}) =  \textnormal{R}(\max_i\bm{z}), ~\rm{when}~\max(\bm{z})>0.
\end{align}
\end{theorem}
Here $\textnormal{R}(\cdot)$ denotes the ReLU activation. The detailed proof is provided in the Appendix. Theorem 2 guarantees that such an operation does not affect the final performance. When deployed on neuromorphic chips or FPGAs, the max pooling operation can be efficiently implemented through a comparator, avoiding additional floating-point multiplications and minimizing power consumption.

\subsection{Channel-wise Threshold Balancing}
During conversion, following~\citet{kim2020spiking, li2021free}, we employ a channel-wise local threshold balancing method applicable to both fully connected and convolutional layers.

In this approach, neurons in each channel share a single threshold for convolutional layers while each neuron has an independent threshold for fully connected layers. Specifically, for convolutional layers, we channel-wisely determine the threshold value, while for fully connected layers, we optimize the threshold element-wise. The update function is given by:
\begin{align}
    \Delta \theta_c^l = - \sum_{i=0}^{N-1} 2(\hat z^l_{c,i}-\theta^l_c)H(\hat z^l_{c,i}-\theta^l_c),
\end{align}
where $c$ denotes the channel index (or element index for fully connected layers) and $i$ denotes the element index within the current channel.

After optimization, the thresholds for each channel can be individually absorbed into the corresponding channel of the convolutional kernel or the weight vectors of the fully connected layers. This channel-wise (or element-wise) operation ensures that the behavior of the IF neuron remains unaffected. This algorithm aims to more precisely determine the optimal threshold value while preserving the fundamental properties of spiking neurons.

\subsection{Delayed Evaluation Strategy}
As introduced in \citep{bu2022optimized}, the first spike from the last layer often takes longer than expected to appear due to propagation delays between layers.
This delay creates significant fluctuations in the firing rate during the initial time-steps, as no spikes are generated until the first one reaches each layer \citep{hwang2021low}. To address this issue, we propose an effective delayed step selection strategy building upon the delayed evaluation technique~\citep{hwang2021low}. 

As shown in Figure~\ref{fig:lws}, the basic idea of the delayed evaluation is to compute the average spike count in the last layer from time-step $t_0$ to $T$, where $t_0$ denotes the delayed step, rather than averaging over all time-steps. Since the output is collected from the average spike count in the interval $[t_0, T]$ at the last layer, we can rewrite Equation~\ref{eq:pot} for the last layer as
\begin{align}
    \bm{r}^L(t_0,T)=\bm{w}^L \bm{r}^{L-1}(t_0,T)-\frac{\bm{v}^L(T)-\bm{v}^L(t_0)}{T-t_0}, \label{eq:pot2}
\end{align}
where $\bm{r}^l(t_0,t)=\sum_{i=t_0}^t \bm{s}^l(i) \theta^l/(t-t_0)$ denotes the average postsynaptic potential from time $t_0$ to $t$, and $L$ represents the last layer index. When $T \gg t_0$, one can still guarantee that the basic equivalence of the average postsynaptic potential and ANN activations holds.

In our strategy, we do not start counting the average output spike until $\widetilde{t_0}$, where $\widetilde{t_0}$ is the expected spike timing of the first spike in the last layer. We estimate $\widetilde{t_0}$ based on the converted SNN model, denoted as
\begin{align}
    \widetilde{t_0} = \sum_{l=1}^{L} \frac{\theta^l- v^l(0)}{\max_i \left (\overline{\text{R}(\hat{\bm z}^l)}\right )}
    \approx \sum_{l=1}^{L} \frac{\theta^l - v^l(0)}{\max_i \left ( \overline{\text{R}(\bm{z}^l)} \right ) } . \label{eq:delay}
\end{align}

Here $\max_i(\cdot)$ selects the maximum element of a given vector. The notation $\overline{(\cdot)} = \mathbb{E}_{x^0 \in\mathcal{D}}[(\cdot)]$ represents the expectation of the given input over dataset. Additionally, we use the ReLU function $\text{R}(\cdot)$ to prevent negative estimated time. Before evaluation, we first calculate $\widetilde{t_0}$ according to Equation~\ref{eq:delay} and consider it as an indicator for the delayed evaluation strategy. 
In practice, if the inference step $T$ is less than $\widetilde{t_0}+4$, we set $t_0=T-4$ and collect the average output over the time interval $[T-4, T]$. If $T$ exceeds $\widetilde{t_0}+4$, we use the interval $[\widetilde{t_0}, T]$ for output estimation. The value $4$ is empirically chosen to ensures a sufficient delay while allowing adequate time for averaging spikes over time.

This technique introduces minimal additional training cost and does not alter the model's inference process, making it seamlessly integrated into the proposed conversion framework without disrupting the overall pipeline.

\section{Experiments}
We conduct extensive experiments to demonstrate the effectiveness of our method and highlight its potential practical value. We first showcase the advantages of our approach by integrating it with existing conversion frameworks, comparing it with state-of-the-art post-training conversion algorithms, and conducting ablation experiments to validate the effectiveness of the whole conversion pipeline. Subsequently, we apply our algorithm to various visual tasks, including image classification, semantic segmentation, object detection, and video classification. 
The feasibility of our algorithm is highlighted on both classification and regression tasks, showcasing its generalizability across different datasets and its superiority over existing conversion algorithms.

To underscore that our conversion method does not require additional training, we utilize open-source pre-trained ANN models for conversion, with most pre-trained sourced from TorchVision~\citep{torchvision2016}. Additionally, we incorporate the membrane potential initialization algorithm~\citep{bu2021optimal}, setting the membrane potential to half of the threshold before inference. The pseudocode for our conversion pipeline and detailed training settings are provided in the Appendix.

\subsection{Plug-and-Play Integration for Enhanced Performance}
\begin{table}[t]
\centering
\renewcommand\arraystretch{0.2}
\setlength\tabcolsep{5pt}
\footnotesize
{
\begin{threeparttable}
\begin{tabular}{ccccccccccc} \toprule
\multirow{3}{*}{Method} & \multirow{3}{*}{ANN} & \multirow{3}{*}{CF} & \multicolumn{5}{c}{Inference Step} \\ \cmidrule{4-8}
& & & 32 & 64 & 128 & 256 & 512 \\ \midrule
\multirow{2}{*}{LCP}  & \multirow{2}{*}{75.66} & $\times$ & 50.21 & 63.66 & 68.89 & 72.12 & 74.17 \\ \cmidrule{3-8}
& & \checkmark & 53.80 & 67.18 & 72.18 & 74.07 & 74.56 \\\midrule
\multirow{2}{*}{ACP}  & \multirow{2}{*}{75.66} & $\times$ & 64.54 & 71.12 & 73.45 & 74.61 & 75.35 \\ \cmidrule{3-8}
& & \checkmark & 65.95 & 72.87 & 74.21 & 74.73 & 75.26\\ \midrule
\multirow{2}{*}{QCFS} & \multirow{2}{*} {73.30} & $\times$ & 64.43 & 71.51 & 73.25 & 73.67 & 73.76  \\ \cmidrule{3-8}
& & \checkmark & 71.94 & 73.06 & 73.50 & 73.49 & 73.54\\ \midrule
\end{tabular}
\end{threeparttable}
}
\caption{Integration of the proposed conversion approach with existing framework}
\label{tab:combine}
\vspace{-5pt}
\end{table}

In this section, we aim to highlight that our proposed conversion method provides a lightweight, plug-and-play solution that can be seamlessly integrated with other conversion algorithms to enhance performance. We select two representative ANN-SNN conversion approaches from existing work: the post-training fine-tuning (calibration) conversion method LCP/ACP~\citep{li2021free, li2024error}, and the training-required conversion method QCFS~\citep{bu2021optimal}. We reproduce the conversion pipeline of both methods, incorporate a local threshold balancing method for threshold searching, and use our delayed evaluation strategy. The evaluation is conducted on the ImageNet dataset with the ResNet-34 architecture.

The results, as detailed in Table~\ref{tab:combine}, showcase the performance of various configurations. The column labeled "CF" indicates whether the proposed method is used. Notably, the composite approaches, which combine our method with either LCP/ACP or QCFS, consistently outperform the individual conversion method across most timesteps. After incorporating our methods, both QCFS and LCP exhibit performance improvements exceeding 3\% when T=32, while maintaining consistent performance as $T$ increases. This demonstrates the advantages of using our approach in conjunction with existing conversion techniques.

\subsection{Ablation Study for the Conversion Framework}
\begin{figure}[t] 
    \centering
    \begin{subfigure}[b]{0.22\textwidth}
        \includegraphics[width=\textwidth]{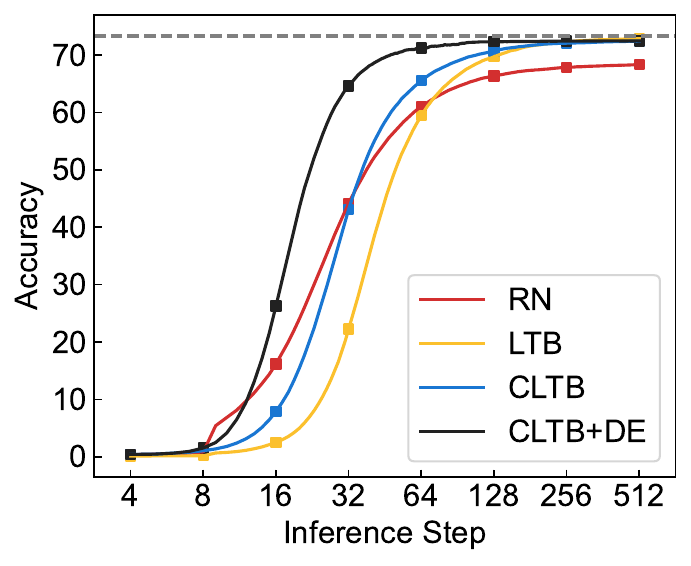}
        \caption{Ablation Study}
        \label{fig:ablation}
    \end{subfigure}
    \hspace{10pt}
    \begin{subfigure}[b]{0.22\textwidth}
        \includegraphics[width=\textwidth]{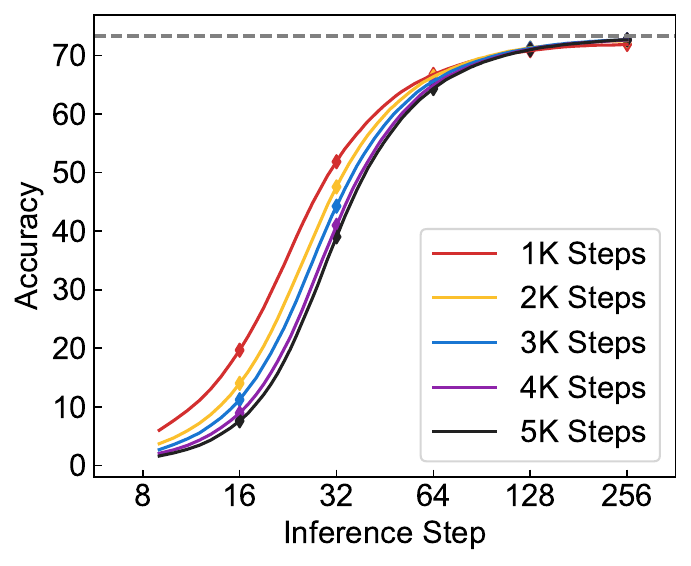}
        \caption{Effect of Iteration Steps}
        \label{fig:steps}
    \end{subfigure}
    \vspace{-5pt}
    \caption{(a) Results on evaluating the impact of different techniques within the conversion framework. (b) Effect of iteration steps on SNN performance after conversion.}
\vspace{-5pt}
\end{figure}

We demonstrate the effectiveness of our proposed conversion framework by using the ResNet-34 architecture on the ImageNet dataset. An ablation study is conducted to evaluate the impact of different techniques within the framework. Specifically, as shown in Figure~\ref{fig:ablation}, we compare the performance of four different SNNs obtained by combinations of different techniques. The baseline SNN is obtained using the Robust-Norm~\citep{rueckauer2017conversion} algorithm (RN, red curve), which is a commonly used post-training ANN-SNN conversion method that sets the 99-th percentile activation value as the threshold. The other SNNs are obtained by using the layer-wise local threshold balancing algorithm (LTB, yellow curve), the channel-wise threshold balancing algorithm (CLTB, blue curve), and a combination of channel-wise threshold balancing and delayed evaluation (CLTB+DE, black curve).

Figure~\ref{fig:ablation} presents the accuracy changes of the converted SNNs with respect to the increasing time-steps. The performance of the SNN obtained from the full conversion pipeline (black curve) consistently outperforms the other SNNs across all time-steps, demonstrating an excellent balance between fast inference and high performance of the proposed conversion framework. This is particularly evident when combined with the proposed delayed evaluation technique. Compared to the robust norm method (red curve), the peak performance of the SNNs obtained by the proposed local threshold balancing method (black, blue and yellow curves) more closely aligns with the original ANN accuracy. Furthermore, the use of the channel-wise threshold balancing (CLTB, blue curve) results in a noticeable performance improvement over the vanilla threshold balancing (LTB, yellow curve) algorithm.

\subsection{Effect of the Iteration Step}
We evaluate the influence of the iteration steps of the local threshold balancing algorithm. The iteration number of the local threshold balancing serves as the hyper-parameter of the method. 
We focus on the most straightforward approach by empirically evaluating the impact of different iteration steps. We converted five different SNNs from the pre-trained ImageNet ResNet-34 model, varying the iterations steps from 1000 to 5000. 
The batch size during local threshold balancing is consistently set to 100. Figure~\ref{fig:steps} shows the final accuracy of the different converted SNNs. As number of iterations increases, the peak performance of the converted SNNs at different inference steps approaches the accuracy of the original ANN. Furthermore, these SNNs demonstrate better performance at low time-steps when fewer iteration steps are used. On the ImageNet dataset, with only 1000 iterations of local threshold update, the performance gap between the SNN at 512 inference time-steps and the ANN is approximately 1\%. With 5000 iterations, this difference is further reduced to around 0.2\%. Notably, the accuracy surpasses 65\% within just 55 time-steps, achieving a balance between inference time and performance.

\subsection{Effect of the Delayed Step}
\begin{figure}[t] 
    \centering
    \includegraphics[width=.46\textwidth]{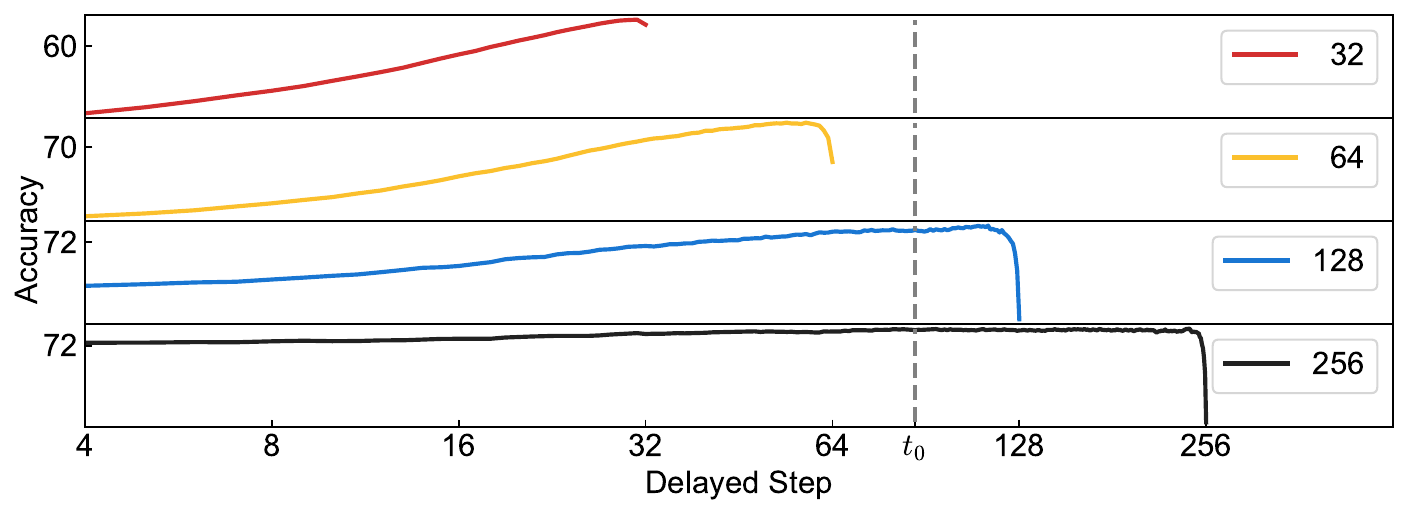}
\vspace{-5pt}
\caption{Effect of delayed steps}
\label{fig:delay}
\end{figure}
\vspace{-5pt}

We conduct experiments to demonstrate the effectiveness of the proposed delayed strategy on the ImageNet dataset using the ResNet-34 architecture. As shown in Figure \ref{fig:delay}, we evaluate the impact of the delayed step across four different inference time-step settings: 32 steps (red curve), 64 steps (yellow curve), 128 steps (blue curve), and 256 steps (black curve). The grey vertical line indicates the estimated $\widetilde{t_0}$ for the currently converted SNN. For each sub-figure, the x-axis represents the number of delayed steps, while the y-axis denotes the accuracy. 

Although the delayed strategy is pre-determined before inference, the estimated $\widetilde{t_0}$ can still serves as an indicator that provides a near-optimal strategy to select the delayed steps. When the total inference step is less than the estimated $\widetilde{t_0}$ (red and yellow curves), the final performance gradually improves with the increase in the delayed steps, but drops significantly when the delayed step approaches the inference step. In the situation that the total inference step exceeds $\widetilde{t_0}$ (blue and black curve), the performance first increases then fluctuates when delayed step is larger than $\widetilde{t_0}$, and ultimately experiences a sharp decline as the delayed steps get close to the inference step. The above results indicate that we can always reach near-optimal performance with the proposed delayed evaluation strategy.

\subsection{Test on Image Classification task}
\begin{table}[t]
\centering
\renewcommand\arraystretch{0.2}
\setlength\tabcolsep{3pt}
\footnotesize
{
\begin{threeparttable}
\begin{tabular}{cccccccc} \toprule
\multirow{2}{*}{Architecture} & \multirow{2}{*}{Method} & \multirow{2}{*}{ANN} & \multicolumn{5}{c}{Inference Step} \\ \cmidrule{4-8}
&  &  & 32 & 64 & 128 & 256 & $\geqslant$512\\ \midrule
\multicolumn{8}{c}{\textbf{Post-Training}} \\ \midrule
ResNet-34\tnote{1} & RMP & 70.64  & - & - & - & 55.65 & 60.08      \\ \midrule
ResNet-34\tnote{2} & LCP & 75.66  & 55.16 & 67.56 & 72.48 & 74.53 & 75.44    \\ \midrule
MobileNet & LCP & 73.40  & 0.20 & 15.47 & 55.95 & 66.35 & 72.19 \\ \midrule
MobileNetV2 & Ours  & 71.88  & 19.39 & 49.93 & 64.31 & 68.89 & 70.48      \\ \midrule
ResNet-18\tnote{3} & Ours   & 69.76  & 65.55 & 68.53 & 69.03 & 69.21 & 69.21\\ \midrule
ResNet-34\tnote{3} & Ours   & 73.31  & 64.88 & 71.17 & 72.30 & 72.42 & 72.46      \\ \midrule
\multicolumn{8}{c}{\textbf{Training-involved}} \\ \midrule
ResNet-34\tnote{3} & QCFS   & 73.30 & 64.43 & 71.51 & 73.25 & 73.67 & 73.76      \\ \midrule
ResNet-34\tnote{2} & ACP    & 75.66 & 64.65 & 71.30 & 73.94 & 75.00 & 75.45      \\ \midrule
ResNet-34\tnote{3} & Ours*  & 73.30 & 71.94 & 73.06 & 73.50 & 73.49 & 73.54      \\ \bottomrule
\end{tabular}
\begin{tablenotes}
    \footnotesize
    \item[1] Standard ResNet-34 architecture without batch normalization.
    \item[2] Two more layers in the first conv-block compared to standard ResNet-34.
    \item[3] Standard ResNet-18/34 architecture.
\end{tablenotes}
\end{threeparttable}
\caption{Comparison of the proposed method and previous works on the ImageNet dataset.}
\label{tab:imagenet}
}
\vspace{-5pt}
\end{table}

We evaluate our conversion method on the classification task with the ImageNet dataset, employing different architectures including ResNet and MobileNet. To highlight the advantages of the proposed conversion algorithm, we conduct comparisons with previous post-training algorithms, showcasing its superior performance. When using the ResNet-34 architecture, the SNN converted by our algorithm exceeds 71\% accuracy within 64 time steps. In contrast, the SNNs converted by the previous algorithm~\citep{han2020deep} and the calibration-required conversion algorithm~\citep{li2024error} require longer time-steps to achieve comparable accuracy.
We also implement the proposed conversion framework on QCFS pre-trained models (denote as "Our*" in Table~\ref{tab:imagenet}), and compare with the sota training-involved ANN-SNN techniques. The accuracy of the proposed method reaches 71.94\% within 32 steps, demonstrating significant advantages compared to the fine-tuning required advanced calibration technique ACP~\citep{li2024error} and training-involved vanilla QCFS~\citep{bu2021optimal}.
The experiments on the ImageNet dataset demonstrate the superior performance of the proposed conversion algorithm and highlight its potential for large-scale applications.

In addition, we demonstrate the energy-saving advantage of our method by theoretically estimating the energy efficiency of our converted SNN. We use the same energy estimation approach as \citet{bu2021optimal}. With ResNet-34 architecture on ImageNet, SNN converted from our algorithm can achieve energy efficiency at 622FPS/W while maintaining 90\% performance. In comparison, the energy efficiency of the original ANN is only 22FPS/W, which is 28 times lower than that of the converted SNN. For detailed energy consumption analysis, please refer to the Appendix.

\subsection{Test on Semantic Segmentation, Object Detection, and Video Classification Tasks}
We further extend the proposed ANN-SNN conversion method to semantic segmentation, object detection, and video classification tasks. For the semantic segmentation task, we evaluate our method on two commonly used datasets: Pascal VOC 2012~\citep{Everingham15} augmented by SBD~\citep{6126343} and MS COCO 2017~\citep{lin2014microsoft}. We employ various semantic segmentation models, including FCN~\citep{long2015fully} and DeepLabV3~\citep{chen2017rethinking}, provided by TorchVision~\citep{torchvision2016} and other open-source repositories. As shown in Table~\ref{tab:seg}, our method is compatible with tackling semantic segmentation tasks without any preparatory training.  Specifically, we achieved a 53.50\% mIoU with the FCN model using a ResNet-34 backbone in 64 time-steps on the Pascal VOC dataset and a 61.52\% mIoU with the DeepLabV3 model using a ResNet-50 architecture in 128 timesteps on the more complex MS COCO 2017 dataset. Given that pixel-level classification tasks require precise model output, this excellent performance highlights the effectiveness and generalizability of the proposed method.

For object detection tasks, we use the benchmark dataset MS COCO 2017~\citep{lin2014microsoft} with open-source models from TorchVision~\citep{torchvision2016}.  We employ the fully convolutional object detection method RetinaNet~\citep{lin2017focal} and FCOS~\citep{tian2019fcos}, and convert the pre-trained ANN models to SNNs using the proposed method. Table~\ref{tab:seg} shows detailed performance under different inference time-steps. Within 64 time-steps, the converted SNN can achieve 27.3\% mAP with RetinaNet and 26.5\% mAP with FCOS, which represents a significant improvement compared to previous explorations~\citep{kim2020spiking, miquel2021retinanet}. The previous method required over 1000 time-steps to achieve comparable performance as ANN, which is detrimental to real-time detection and energy efficiency.

For the video classification tasks, we use a small Kinetics-400 dataset~\citep{kay2017kinetics}, which is a subset of the original dataset. More specifically, we split the original validation set into two independent sets, one for our threshold balancing procedure, accounting for about 60\% of the original validation set, and the other for the test. As a result, the new train set contains 12000 videos and the new test set contains 7881 videos. 
The converted MC3-18~\citep{DBLP:journals/corr/abs-1711-11248} SNN model achieves 61.60\% accuracy within 64 time-steps, showing great potential of our method in more complex computer vision tasks.

\begin{table}[t]
\centering
\renewcommand\arraystretch{0.2}
\setlength\tabcolsep{3pt}
\footnotesize
{
\begin{threeparttable}
\begin{tabular}{cccccccccc} \toprule
\multirow{3}{*}{Dataset} & \multirow{3}{*}{Arch.}  & \multirow{3}{*}{ANN} & \multicolumn{4}{c}{Inference Step} \\ \cmidrule{4-7}
& &   & 32 & 64 & 128 & 256 \\ \midrule
\multicolumn{7}{c}{\textbf{Semantic Segmentation} (mIoU\%)} \\ \midrule
\multirow{17}{*}{PascalVOC}
&FCN-18       & 47.30    & 43.49     & 47.02     & 47.26    & 47.23    \\ \cmidrule{2-7} 
&FCN-34       & 54.99    & 46.60     & 53.50     & 54.75    & 54.76    \\ \cmidrule{2-7} 
&DeepLab-18  & 55.75     & 48.59     & 53.91     & 55.40    & 55.60     \\ \cmidrule{2-7} 
&DeepLab-34  & 58.95     & 41.20     & 55.74     & 58.36    & 58.68     \\ \midrule 
\multirow{3}{*}{MS COCO}
&FCN-50  & 60.67    & 28.66     & 50.38     & 57.01     & 59.53      \\ \cmidrule{2-7} 
&DeepLab-50  & 63.01    & 39.02     & 56.62     & 61.52     & 62.61      \\ \midrule 
\multicolumn{7}{c}{\textbf{Object detection} (mAP\%)} \\ \midrule
\multirow{3}{*}{MS COCO}
&RetinaNet-50     & 36.4  & 16.6 & 27.3 & 32.2 & 34.3   \\ \cmidrule{2-7} 
&FCOS-50          & 39.2  & 12.2 & 26.5 & 32.1 & 33.6  \\ \midrule 
\multicolumn{7}{c}{\textbf{Video Classification} (Acc\%)} \\ \midrule
\multirow{2}{*}{S-Kinetics-400}
&MC3-18          & 63.52  & 55.21 & 61.60 & 62.43 & 62.73  \\ \cmidrule{2-7} 
&R3D-18          & 62.56  & 53.46 & 60.31 & 61.12 & 61.15  \\ \midrule
\end{tabular}
\end{threeparttable}
}
\caption{Performance on semantic segmentation (mIoU\%), object detection (mAP\%), and video classification (Acc\%) tasks.}
\label{tab:seg}
\vspace{-5pt}
\end{table}

\section{Conclusion and Limitation}
This paper introduces the concept of a lightweight post-training ANN-SNN conversion algorithm with inference-scale computational cost, which can directly convert pre-trained ANN models into SNNs without GPU-based training. The significant reduction in training overhead highlights the potential of the rapid deployment of large-scale SNNs in various scenarios and makes the instant on-chip conversion from pre-trained ANN to SNN possible.
However, there is a trade-off between performance and training cost, with the overall performance of the proposed post-training framework being slightly weaker compared to other training-involved SNN learning methods. Additionally, while this work focuses on the conversion of convolutional SNNs, future efforts will aim to extend such a conversion approach to transformer architectures and other complex layers.

\section{Acknowledgement}
\noindent
This work is funded by National Natural Science Foundation of China (62422601,62176003,62088102), Beijing Municipal Science and Technology Program (Z241100004224004), Beijing Nova Program (20230484362).

\clearpage
{
    \small
    \bibliographystyle{ieeenat_fullname}
    \bibliography{main}
}

\appendix

\renewcommand{\thefigure}{S\arabic{figure}}  
\renewcommand{\theequation}{S\arabic{equation}} 
\renewcommand{\thetable}{S\arabic{table}}

\setcounter{section}{0}
\setcounter{theorem}{0}
\setcounter{equation}{0}
\setcounter{figure}{0}

\maketitlesupplementary

\section{Proof for Error Bound}
\begin{theorem*}
\label{t1}
The layer-wise conversion error can be divided into intra-layer and inter-layer errors:
\begin{align}
    e^l &\leqslant \overbrace{\lVert \textnormal{S}(\bm{\hat{z}}^l; \theta^l) - \textnormal{A}(\bm{\hat{z}}^l)\rVert_2}^\textnormal{intra-layer~error} + \overbrace{\lVert \bm{w}^l \rVert_2 e^{l-1}}^\textnormal{inter-layer error}.
\end{align}
Given that both ANN and SNN models receive the same input in the first layer, leading to $e^0 = 0$, the upper bound for the conversion error between ANN and SNN models in an $L$-layer fully-connected network is given by
\begin{align}
    e_\textnormal{model}=e^L &\leqslant \sum_{l=1}^{L} \left (\prod_{k=l+1}^{L} \lVert \bm{w}^k \rVert_2\right ) \left \lVert \textnormal{S}(\bm{\hat{z}}^l; \theta^l) - \textnormal{A}(\bm{\hat{z}}^l) \right \rVert_2
\end{align}
\end{theorem*}

\begin{proof}[error bound]
According to the definition of the conversion error (Equation (7)), we have
\begin{align}
    e^l &= 
    \lVert \textnormal{S}(\bm{\hat{z}}^l) - \textnormal{A}(\bm{z}^l)\rVert_2  \nonumber\\
    &=
    \lVert \textnormal{S}(\bm{\hat{z}}^l) - \textnormal{A}(\bm{\hat{z}}^l) +
    \textnormal{A}(\bm{\hat{z}}^l) - \textnormal{A}(\bm{z}^l)\rVert_2  \\
    &\leqslant 
    \lVert \textnormal{S}(\bm{\hat{z}}^l) - \textnormal{A}(\bm{\hat{z}}^l)\rVert_2 +
    \lVert \textnormal{A}(\bm{\hat{z}}^l) - \textnormal{A}(\bm{z}^l)\rVert_2 \nonumber
\end{align}
Here the ANN activation function is  defined a s $\textnormal{A}(\cdot) = \textnormal{R}(\cdot) = \textnormal{ReLU}(\cdot)$, where $\textnormal{ReLU}(\cdot)$ is ReLU function. We first prove that $\lVert \textnormal{R}(\bm{\hat{z}}^l) - \textnormal{R}(\bm{z}^l)\rVert_2 \leqslant \lVert (\bm{\hat{z}}^l - \bm{z}^l)\rVert_2$. To do so, we analyze four possible cases for $z^l_i$ and $\hat z^l_i$, which are individual elements of $\bm{z}^l$ and $\bm{\hat z}^l$, respectively.
\begin{align}
    & \textnormal{if}~\hat z^l_i \geqslant 0,~z^l_i \geqslant 0, \\ \nonumber
    & \quad \quad \quad \textnormal{then}~(\textnormal{R}(\hat z^l_i) - \textnormal{R}(z^l_i))^2=(\hat z^l_i - z^l_i)^2 \\ \nonumber
    & \textnormal{if}~\hat z^l_i \geqslant 0,~z^l_i \leqslant 0,\\ \nonumber
    & \quad \quad \quad \textnormal{then}~(\textnormal{R}(\hat z^l_i) - \textnormal{R}(z^l_i))^2=(\hat z^l_i - 0)^2 \leqslant (\hat z^l_i - z^l_i)^2\\ \nonumber
    & \textnormal{if}~\hat z^l_i \leqslant 0,~z^l_i \geqslant 0, \\ \nonumber
    & \quad \quad \quad \textnormal{then}~(\textnormal{R}(\hat z^l_i) - \textnormal{R}(z^l_i))^2=(0 - z^l_i)^2 \leqslant (\hat z^l_i - z^l_i)^2 \\ \nonumber
    & \textnormal{if}~\hat z^l_i \leqslant 0,~z^l_i \leqslant 0,\\ \nonumber
    & \quad \quad \quad \textnormal{then}~(\textnormal{R}(\hat z^l_i) - \textnormal{R}(z^l_i))^2=(0 - 0)^2 \leqslant (\hat z^l_i - z^l_i)^2
\end{align}
Therefore, for each element in vector $\bm{z}^l$ and $\bm{\hat z}^l$, we can conclude that $\forall i, (\textnormal{A}(\hat z^l_i) - \textnormal{A}(z^l_i))^2 \leqslant (\hat z^l_i - z^l_i)^2$. From this, we can further derive 
\begin{align}
\lVert \textnormal{R}(\bm{\hat{z}}^l) - \textnormal{R}(\bm{z}^l)\rVert_2 \leqslant \lVert (\bm{\hat{z}}^l - \bm{z}^l)\rVert_2.
\end{align}
Back to the main theorem, we further rewrite the conversion error bound as
\begin{align}
    e^l &\leqslant 
    \lVert \textnormal{S}(\bm{\hat{z}}^l) - \textnormal{A}(\bm{\hat{z}}^l)\rVert_2 +
    \lVert (\bm{\hat{z}}^l - \bm{z}^l)\rVert_2 \nonumber\\
    &\leqslant 
    \lVert \textnormal{S}(\bm{\hat{z}}^l) - \textnormal{A}(\bm{\hat{z}}^l)\rVert_2 +
    \lVert \bm{w}^l (\textnormal{S}(\bm{\hat{z}}^{l-1}) - \textnormal{A}(\bm{z}^{l-1}))\rVert_2 \label{A:eq22}\\
    &\leqslant 
    \lVert \textnormal{S}(\bm{\hat{z}}^l) - \textnormal{A}(\bm{\hat{z}}^l)\rVert_2 +
    \lVert \bm{w}^l \rVert_2 \lVert \textnormal{S}(\bm{\hat{z}}^{l-1}) - \textnormal{A}(\bm{z}^{l-1}) \rVert_2 \label{A:eq23}\\
    &\leqslant 
    \overbrace{\lVert \textnormal{S}(\bm{\hat{z}}^l) - \textnormal{A}(\bm{\hat{z}}^l)\rVert_2}^\textnormal{intra-layer~error} + \overbrace{\lVert \bm{w}^l \rVert_2 e^{l-1}}^\textnormal{inter-layer error}. \label{A:eq231}
\end{align}
Note that $\lVert \bm{w}^l \rVert_2$ in Equation~\ref{A:eq23} represents the matrix norm (p=2) or spectral norm of the weight matrix $\bm{w}^l$, and the derivation from Equation~\ref{A:eq22} to \ref{A:eq23} holds true because of the property of the spectral norms. From the inequality above, we can find that the layer-wise conversion error is bounded by two components: the intra-layer error, which is layer-wise error when both the analog and spiking neurons receive the same input, and the inter-layer error, which is proportional to the layer-wise error in the previous layer.

We further derive the conversion error between models, which corresponds to the conversion error in the last output layer. For simplicity, we define the intra-layer error in each layer as $\varepsilon^l$. According to Equation~\ref{A:eq231}, we get
\begin{align}
    e^L &\leqslant\lVert \textnormal{S}(\bm{\hat{z}}^L) - \textnormal{A}(\bm{\hat{z}}^L)\rVert_2 + \lVert \bm{w}^L \rVert_2 e^{L-1} \nonumber \\
    &= \varepsilon^L + \lVert \bm{w}^L \rVert_2 e^{L-1}.
\end{align}
Also, since we use direct input coding for SNNs, there is no conversion error in the $0$-th layer, and the conversion error in the first layer is given by $e^1 = \varepsilon^1 = \lVert \text{S}(\bm{\hat{z}}^1) - \text{A}(\bm{\hat{z}}^1)\rVert_2$. 
By iteratively applying this relationship across layers, we can derive the error bound for arbitrary layer. The error bound for the final output should be
\begin{align}
    e^L &\leqslant \varepsilon^L + \lVert \bm{w}^L \rVert_2 e^{L-1} \nonumber\\
    &\leqslant \varepsilon^L + \lVert \bm{w}^L \rVert_2 \varepsilon^{L-1} + \lVert \bm{w}^L \rVert_2 \lVert \bm{w}^{L-1} \rVert_2 e^{L-2}\nonumber\\
    &\leqslant \varepsilon^L + \lVert \bm{w}^L \rVert_2 \varepsilon^{L-1} + ... + \lVert \bm{w}^L \rVert_2 ... \lVert \bm{w}^{2} \rVert_2 \varepsilon^{1} \nonumber \\
    &=
    \sum_{l=1}^{L} \left(\prod_{k=l+1}^{L} \lVert \bm{w}^k \rVert_2 \right) \varepsilon^l,~~\left( \text{Define}~\prod_{k=L+1}^{L} \lVert \bm{w}^k \rVert_2 =1 \right)\nonumber \\
    &= \sum_{l=1}^{L} \left (\prod_{k=l+1}^{L} \lVert \bm{w}^k \rVert_2\right ) \left \lVert \textnormal{S}(\bm{\hat{z}}^l; \theta^l) - \textnormal{A}(\bm{\hat{z}}^l) \right \rVert_2
\end{align}
\end{proof}

\section{Proof for Update Rule}
The final update rule for the local threshold balancing algorithm at each step is:
\begin{align}
    \Delta \theta^l &=  - \sum_{i=1}^{N} 2(\hat z^l_i-\theta^l)H(\hat z^l_i-\theta^l), \\
    \theta^l &\leftarrow \theta^l - \eta \Delta \theta^l.
\end{align}
\begin{proof}
    As we have mentioned in the main text, our goal is to optimize the following equation:
    \begin{align}
        \forall~l,~\arg \min_{\theta^l} \left (\prod_{k=l+1}^{L} \lVert \bm{w}^k \rVert_2\right ) \left \lVert \textnormal{C}(\bm{\hat{z}}^l; \theta^l) - \textnormal{A}(\bm{\hat{z}}^l) \right \rVert_2^2.
    \end{align}
    We can apply the gradient descent method to iteratively update the threshold value by subtracting the first-order derivative with respect to the threshold, given by:
    \begin{align}
    \Delta \theta^l 
        &= \frac{\partial \left (\prod_{k=l+1}^{L} \lVert \bm{w}^k \rVert_2\right ) \left \lVert \textnormal{C}(\bm{\hat{z}}^l; \theta^l) - \textnormal{A}(\bm{\hat{z}}^l) \right \rVert_2^2}{\partial \theta^l}.
    \end{align}
    Considering each element $\hat{z}^l_i$ in the vector $\bm{\hat z}^l$, for each $i$, we have:
    \begin{align}
        & \frac{\partial \left (\prod_{k=l+1}^{L} \lVert \bm{w}^k \rVert_2\right ) \left ( \textnormal{C}(\bm{\hat{z}}_i^l; \theta^l) - \textnormal{A}(\bm{\hat{z}}_i^l) \right )^2}{\partial \theta^l} \\
        &= \left\{
        \begin{aligned}
        - \left (\prod_{k=l+1}^{L} \lVert \bm{w}^k \rVert_2\right) \cdot 2(\hat z^l_i-\theta^l) \quad \text{if}~\hat z^l_i>\theta^l\\
        0 \quad\text{if}~\hat z^l_i \leqslant \theta^l 
        \end{aligned}
        \right. \nonumber\\
        &= - \left (\prod_{k=l+1}^{L} \lVert \bm{w}^k \rVert_2\right) \cdot 2(\hat z^l_i-\theta^l)H(\hat z^l_i-\theta^l) \nonumber
    \end{align}
    Therefore, consider the derivative for $\theta^l$ over the whole vector with $N$ elements in total, we have
    \begin{align}
        \Delta \theta^l
        &= \frac{\partial \left (\prod_{k=l+1}^{L} \lVert \bm{w}^k \rVert_2\right ) \sum_{i=1}^{N} \left ( \textnormal{C}(\bm{\hat{z}}_i^l; \theta^l) - \textnormal{A}(\bm{\hat{z}}_i^l) \right )^2}{\partial \theta^l} \\
        &= - \left (\prod_{k=l+1}^{L} \lVert \bm{w}^k \rVert_2\right ) \sum_{i=1}^{N} 2(\hat z^l_i-\theta^l)H(\hat z^l_i-\theta^l).\nonumber
    \end{align}
    Since $\left (\prod_{k=l+1}^{L} \lVert \bm{w}^k \rVert_2\right )$ is a constant with fixed weight matrix, we incorporate this term into the learning rate parameter $\eta$. Consequently, the final update rule can be derived as:
    \begin{align}
        \Delta \theta^l &=  - \sum_{i=1}^{N} 2(\hat z^l_i-\theta^l)H(\hat z^l_i-\theta^l), \\
        \theta^l &\leftarrow \theta^l - \eta \Delta \theta^l.
    \end{align}
\end{proof}

\section{Proof for Pre-Neuron Max pooling Layer}
\begin{theorem*}
The order of max pooling layer and ReLU activation layer does not affect the output results.
\begin{align}
    \max \textnormal{R}(\bm{z}) =  \textnormal{R}(\max\bm{z}), ~\textnormal{when}~\max(\bm{z})>0.
\end{align}
\end{theorem*}
\begin{proof}
Since $\textnormal{R}(x)=\max(\bm{x}, 0)$, we can rewrite the left hand side as $\max \left (\textnormal{R}(\bm{z})\right) = \max (\max(\bm{z}, \bm{0})) = \max(\bm{z})$. Similarly, the right-hand side can be written as $\max(\max(\bm{z}), 0) = \max(\bm{z})$, which is equal to the left-hand side. 
\end{proof}

\section{Details for Experiments}

\subsection{Pseudo-code for Full Conversion Pipeline}
In this section, we present the pseudo-code of the full conversion pipeline in Algorithm~\ref{algo:training}. At the start of the conversion process, the model is initialized by replacing all activation layers with clipping function C$(\cdot; \bm\theta^l)$ and the initial threshold values $\bm\theta^l$ for each layer are set to 0. Additionally, all max pooling layers are replaced with pre-neuron max pooling layers, and all other modules are set to inference mode.

During local threshold balancing process, input images are sampled from the training dataset at each iteration and fed into the model. The threshold values can be optimized during forward propagation without global backpropagation. After threshold value optimization, the delayed time is calculated by running another iteration of the model with sampled images from the dataset.

The pseudo-code of the SNN inference process with delayed evaluation technique is presented in Algorithm~\ref{algo:inf}. The delayed time is determined based on the given inference time and the estimated delay time. After $t_0$, the outputs of SNN model are accumulated and the average output value is used as the final prediction.

\begin{algorithm}[t]
	\caption{Efficient ANN-SNN Conversion Algorithm}
	\label{algo:training}
    \textbf{Input}: \\
    ANN model pre-trained weight $\bm{w}$; \\
    Training dataset $\mathcal{D}$;\\
    Iteration steps $K$;\\
    Learning rate $\eta$;\\
    \textbf{Output}: \\
    $\text{SNN}(\cdot;\bm{w}, \bm{\theta})$ \\
    Delayed time $t_0$;
    \begin{algorithmic}[1]
    \STATE \textit{// Initialize model}
        \FOR{$l = 1$ to $L$}
            \STATE Set all activation layer as C($\cdot; \theta^l$)
            \STATE Set pre-neuron max-pooling layer
            \STATE Set the initial value of threshold $\theta^l=0$
            \STATE Set initial weights for SNN as ANN pre-trained weights $\bm{w}$
        \ENDFOR
        \STATE
        \STATE \textit{// Threshold balancing algorithm}
        \STATE step = 0
        \WHILE{$\text{step++} < K$}
            \STATE Sample input images $\bm{x}^0$ in Dataset $D$
            \FOR{$l = 1$ to $L$}
                \STATE $\bm{z}^l$ = $\bm{w}^l \bm{x}^{l-1}$
                \STATE $\bm{x}^l$ = C($\bm{z}^l; \theta^l$)
                \STATE $\Delta \theta^l = - \sum_{i=1}^{N} 2(z^l_i-\theta^l)H(z^l_i-\theta^l)$
                \STATE $\theta^l \leftarrow \theta^l - \eta \Delta \theta^l$
            \ENDFOR
        \ENDWHILE
        \STATE
        \STATE \textit{// Delayed time calculation}
        \STATE Sample input images $\bm{x}^0$ in Dataset $D$
        \STATE $t_0 = 0$
        \FOR{$l = 1$ to $L$}
            \STATE $\bm{z}^l$ = $\bm{w}^l \bm{x}^{l-1}$
            \STATE $
            t_0 = t_0 + \left (\theta^l - v^l(0) \right )/ \left(\max_i\left( \overline{\text{R}(\bm{z}^l)}\right)\right)$
            \STATE $\bm{x}^l$ = C($\bm{z}^l; \theta^l$)
        \ENDFOR
        \STATE
        \STATE \textit{// Initialize SNN model}
        \FOR{$l = 1$ to $L$}
        \STATE $\bm{v}^l(0)\leftarrow \theta^l/2$
        \ENDFOR
        \RETURN $\text{SNN}(\cdot; \bm{w}, \bm{\theta}), t_0$
    \end{algorithmic}
\end{algorithm}

\begin{algorithm}[t]
	\caption{SNN Inference with Delayed Evaluation Strategy}
	\label{algo:inf}
    \textbf{Input}: \\
    SNN model $\text{SNN}(\cdot; \bm{w}, \bm{\theta})$; \\
    Input Image $\bm{x}^0$;\\
    Inference steps $T$;\\
    Delayed time $t_0$;\\
    \textbf{Output}: \\
    Prediction $\bm{o}$;
    \begin{algorithmic}[1]
    \STATE \textit{// SNN inference}
    \STATE step = 0
    \IF{$t_0 > T-4$}
    \STATE $t_0 = T - 4$
    \ENDIF
    \STATE $\bm{o} = \bm 0$
    \FOR{$t = 1$ to $T$}
        \FOR{$l = 1$ to $L$}
            \STATE $\bm{z}^l$ = $\bm{w}^l \bm{x}^{l-1}$
            \STATE $\bm{x}^l$ = S($\bm{z}^l; \theta^l$)
        \ENDFOR
        \IF{$t > t_0$}
        \STATE $\bm{o} = \bm{o} + \bm{x}^l$
        \ENDIF
    \ENDFOR
    \STATE $\bm{o} = \bm{o}/(T-t_0)$
    \STATE
    \STATE \textit{// Reset SNN model}
    \FOR{$l = 1$ to $L$}
    \STATE Reset S($\cdot; \theta^l$)
    \STATE $\bm{v}^l(0)\leftarrow \theta^l/2$
    \ENDFOR
    \RETURN $\bm{o}$
    \end{algorithmic}
\end{algorithm}

\subsection{Image Classification}
When conducting experiments on the ImageNet dataset, we use the pre-trained models from TorchVision. During both the threshold balancing process and inference, we normalize the image to standard Gaussian distribution and Crop the image to size 224$\times$224. The iteration step number of the threshold balancing process is 1000 unless mentioned.

\subsection{Semantic Segmentation}
For the experiments of Pascal VOC 2012 dataset, the weights of the original ANN models are from open-source Github repositories. The data preprocessing operations during both the threshold balancing process and inference process include resizing the data into 256$\times$256 image and normalizing the data value. The delayed evaluation step length is set to half of the inference step length. The iteration step number of the threshold balancing process is set to 4 traversals of the training set for FCN and 5 traversals of the training set for DeepLab. Moreover, Pascal VOC 2012 is augmented by the extra annotations provided by SBD, resulting in 10582 training images.

For the experiments of the MS COCO 2017 dataset, the weights are directly downloaded from TorchVision. When performing data preprocessing, We first resize the input data into 400$\times$400 images and normalize the images. The iteration step number of the threshold balancing process is set to 3 traversals of the training set. Note that these weights were trained on a subset of MS COCO 2017, using only the 20 categories that are present in the Pascal VOC dataset. This subset contains 92518 images for training.

\subsection{Object Detection}
For our object detection experiments, we utilized pre-trained weights from TorchVision. During the threshold balancing process, we use similar dataset augmentation as SSD \citep{liu2016ssd}. The input images are augmented by RandomPhotometricDistort, RandomZoomOut, RandomIoUCrop, and RandomHorizontalFlip. The iteration steps are set to 5000 for each model. During the evaluation of converted models, we directly use normalized images as inputs.

\subsection{Video Classification}
The pre-trained weights for video classification tasks are directly downloaded from TorchVision. We split the original validation set of Kinetics-400 into a new training set and a new test set. The resulting training set contains 12000 videos and the new test set contains 7881 videos. The accuracies are estimated on video-level with parameters frame\_rate=15, clips\_per\_video=5, and clip\_len=16. The frames are resized to 128$\times$171, followed by a central crop resulting into 112$\times$112 normalized frames.

\section{Visualizations on Object Detection and Semantic Segmentation}
In Figure~\ref{fig:demo_det} and Figure~\ref{fig:demo_seg} we present the visualization of the semantic segmentation and object detection results. In each row of the figure, we illustrate the visualization of ground truth, original image (only for semantic segmentation), results from original ANN and results from converted SNN at different time-steps.
\begin{figure*}[t]
\centering
\includegraphics[width=0.99\textwidth]{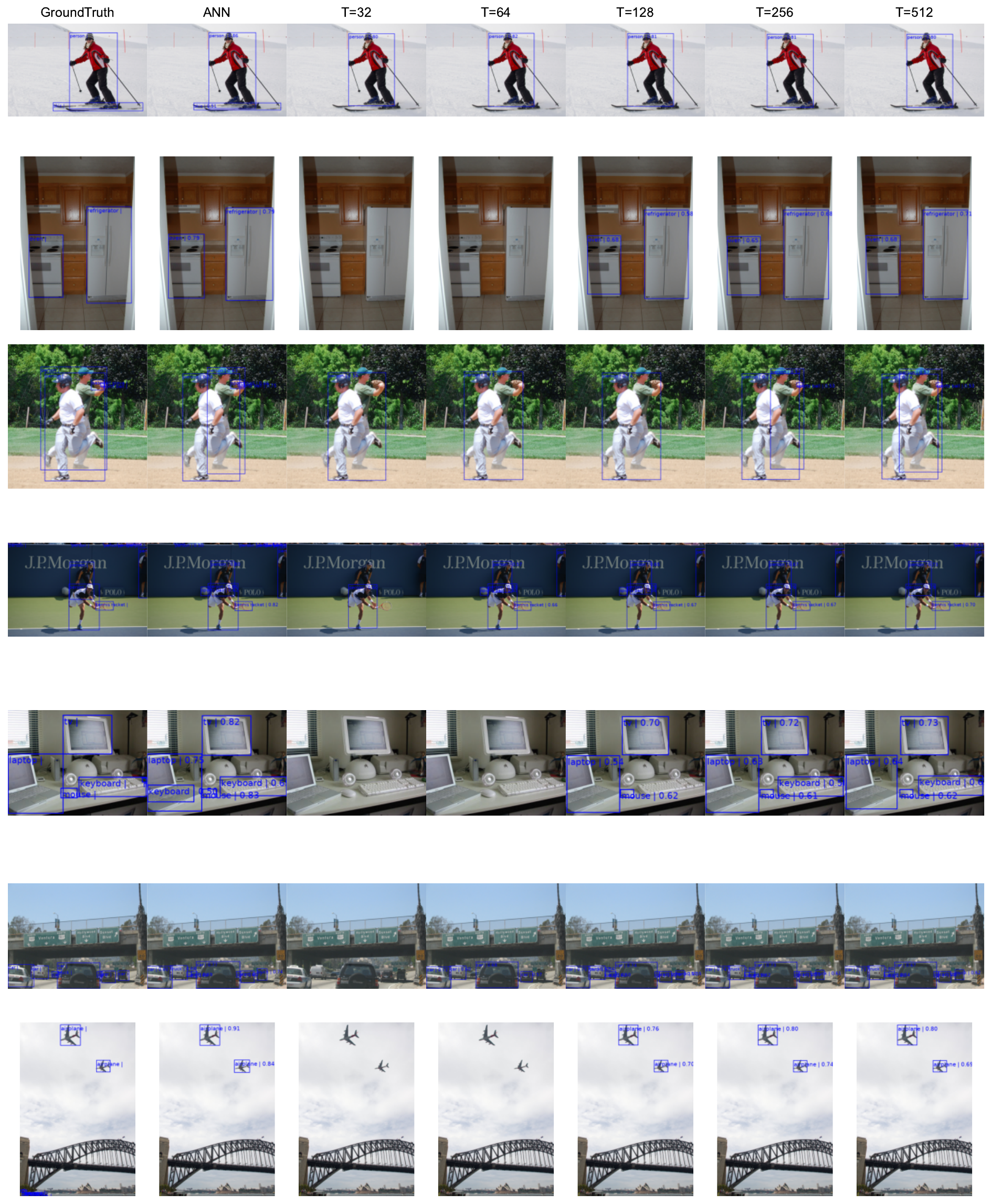}
\caption{Illustration for detection examples of SNNs on different inference steps}
\label{fig:demo_det}
\end{figure*}

\begin{figure*}[t]
\centering
\includegraphics[width=0.99\textwidth]{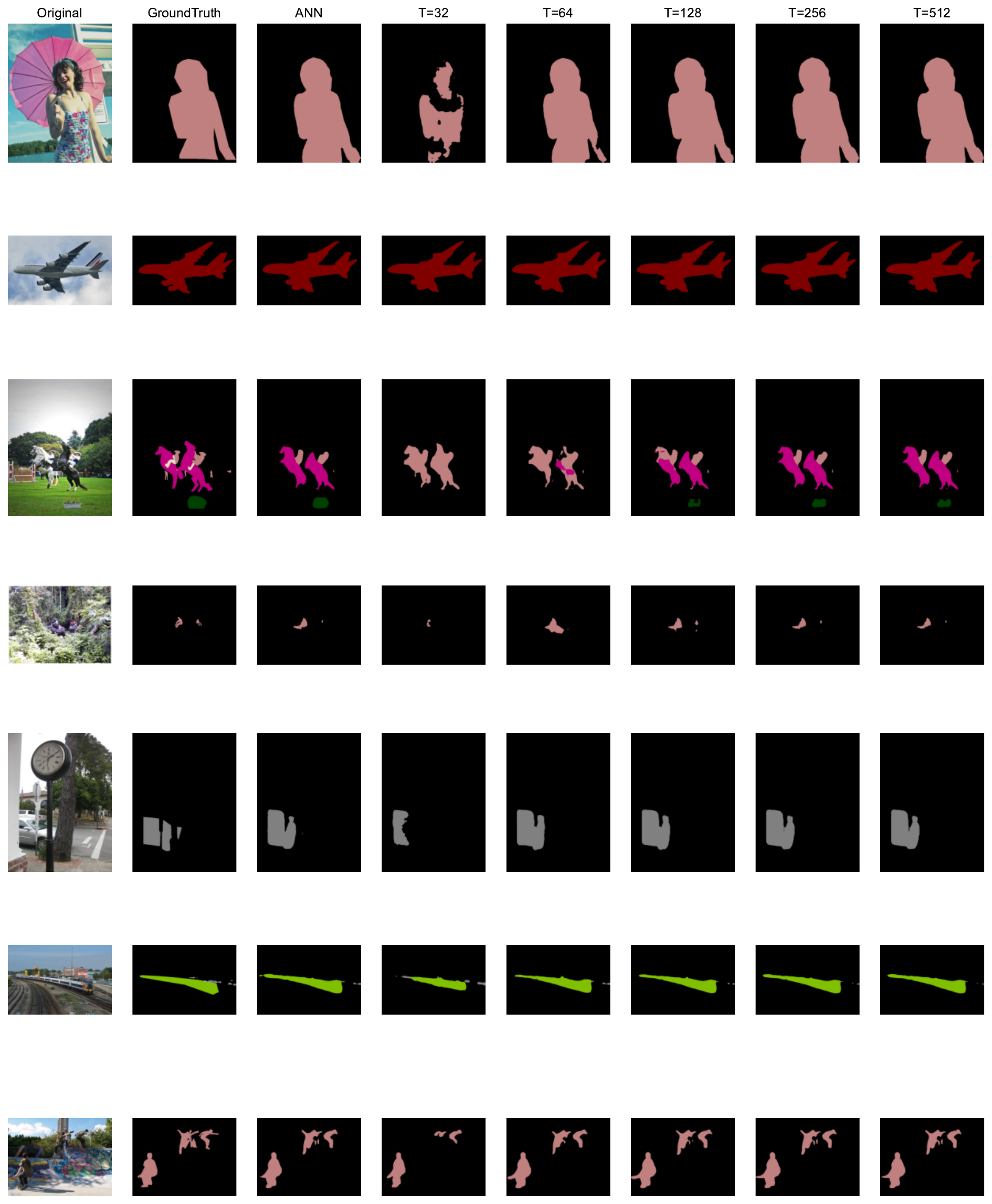}
\caption{Illustration for segmentation examples of SNNs on different inference steps}
\label{fig:demo_seg}
\end{figure*}

\section{Energy Consumption Analysis}
Since the low-power consumption is one of the advantages of SNNs, we calculate the average energy consumption of the converted SNNs and compare it with the energy consumption of the ANN counterparts. We employed a method similar to previous work, estimating the energy consumption of the SNN by calculating the number of Synaptic Operations (SOP). Since the total spike activity of the SNN increases proportionally with the inference time, we define SOP90 and SOP95 as the metrics for converted SNNs on ImageNet. The SOP90/95 denotes the average synaptic operation per image when the accuracy of the converted SNN exceeds 90\%/95\% of the original ANN while SNN90-FPS/W denotes the total number of frames per joule when the performance of the converted SNN exceeds 90\%. In order to further estimate the energy consumption, we utilize the average energy efficiency of 77fJ/SOP for SNN and 12.5pJ/FLOP for ANN \citep{qiao2015reconfigurable} to calculate the required energy for one frame. The detailed comparison is demonstrated in the table below.
\begin{table}[h]
\centering
\setlength\tabcolsep{3pt}
\footnotesize
{
\begin{threeparttable}
\begin{tabular}{ccccccc} \toprule
Architecture & SOPs90 & SOPs95 & ANN-FPS/W & SNN90-FPS/W  \\ \midrule
ResNet-34   & 20.86  & 33.42  & 22        & 622           \\ \midrule
\end{tabular}
\end{threeparttable}
}
\caption{Energy consumption estimation on ImageNet dataset}
\end{table}
For ImageNet classification tasks, using the same ResNet-34 architecture, the SNN is over 28 times more energy efficient than the original ANN while maintaining 90\% performance of the original ANN, achieving an estimation of 622 FPS/W energy efficiency while deployed on neuromorphic hardware. It is worth noticing the SNN can be easily obtained by converting open-source pre-trained ANN models with negligible training cost and then deploy on specific hardware for energy-saving purpose.

\section{Detailed Discussion on Training Cost}
Besides general discussion of the inference-scale complexity of the overall conversion framework, we also demonstrate the efficiency of our method by comparing the total time required for three different post-training conversion methods, including LCP, ACP~\citep{li2024error}, and our method. For all methods, we used the same environment with a 4090 GPU for the evaluation. Here we present the results using ResNet-34 architecture on ImageNet in the below table.

\begin{table}[h]
\centering
\setlength\tabcolsep{3pt}
\footnotesize
{
\begin{threeparttable}
\begin{tabular}{cccccccc} \toprule
Method & T=32   & T=64  & T=128 & T=256 & T=512  \\ \midrule
Ours   & 92.73  & 92.73 & 92.73 & 92.73 & 92.73  \\ \midrule
LCP    & 70.23  & 114.19 & 201.82 & 376.98 & 727.66 \\ \midrule
ACP    & 829.49 & 1218.93 & 2018.56 & 3731.53 & 7059.09 \\ \midrule
\end{tabular}
\end{threeparttable}
}
\caption{Comparison of conversion time with different post-training methods}
\end{table}

It is worth noting that one only needs to run the threshold optimization algorithm once in our method and the obtained SNN can be applied with any simulation time-steps. While in LCP/ACP, calibration is required for every simulation steps. Therefore, in the table above, our algorithm only takes 93 seconds to get the converted SNN and is applicable to any time-steps. Although LCP has an advantage at 32 steps, the required conversion time still increases with the total time-steps and loses its advantage at 64 steps. Moreover, ACP takes much more time because it involves weight updates. This is because our method only requires a similar training cost as ANN inference, which is lower than the computational cost of SNN inference and weight update required in LCP/ACP. 

\end{document}